\pdfoutput=1
\documentclass[sigconf,nonacm]{acmart}

\renewcommand\footnotetextcopyrightpermission[1]{}

\usepackage{colortbl}

\newcommand{\ours}{\textsc{Prism-1}}

\begin{document}

\title{Prism-1: Position- and Rival-Independent Scoring for System One Models}

\author{Yao Long Teng}
\authornote{Both authors contributed equally to this research.}
\affiliation{%
  \institution{Nanyang Technological University}
  \country{Singapore}}

\author{Jianguo Huang}
\authornotemark[1]
\affiliation{%
  \institution{Nanyang Technological University}
  \country{Singapore}}

\author{Jiayi Cai}
\affiliation{%
  \institution{Nanyang Technological University}
  \country{Singapore}}

\author{Yanchen Deng}
\authornote{Corresponding author.}
\affiliation{%
  \institution{Nanyang Technological University}
  \country{Singapore}}

\author{Bo An}
\affiliation{%
  \institution{Nanyang Technological University}
  \country{Singapore}}

\renewcommand{\shortauthors}{Teng, Huang, et al.}

\begin{abstract}
Given a state and a set of options, such as intents, tools or interface elements, System One decision models return a distribution over the options in one forward pass, without generated reasoning. Web systems increasingly build them by reading decisions from language models. Two axioms of rational choice are natural requirements for such decisions, because in web systems the listing and the option set are determined by other components and can change or be manipulated: the order in which the options are listed should not matter, and adding an option should not change the relative odds of the others (Luce's choice axiom). We test the standard reads, namely single-choice question (SCQ) prompting, with shuffled-order training, order averaging or sorted options, and invariant set architectures, and find that each violates at least one axiom: one shuffle changes 12--20\% of SCQ answers, sorting only freezes an arbitrary convention, and a well-chosen added option steers 4--34\% of decisions. By Luce's theorem, a read satisfies both axioms only if it is a softmax over scores that each depend on one option alone, so in a transformer they hold by construction only if no token that scores an option sees another. We implement this as \ours{} (\textbf{P}osition- and \textbf{R}ival-\textbf{I}ndependent \textbf{S}coring for System \textbf{O}ne \textbf{M}odels): every option is read in its own block at shared positions, with the model's usual left-to-right attention, in one forward pass. On five benchmark sets with Qwen3-8B and Gemma 4 12B, SCQ prompting trained on the same data still changes 7\% of its answers under one shuffle, and an added irrelevant option steers 12\% of its decisions. \ours{} stays within one point of that baseline's accuracy and is 2--4 points more accurate once only answers that survive a reordering are counted, while no shuffle changes any of its answers and no added irrelevant option moves its choice among the others.
\end{abstract}



\maketitle

\section{Introduction}

System One models answer typed questions in a single forward pass~\cite{typesafe2026docs,almeida2026jev}. Named after the fast, intuitive mode of human thinking~\cite{kahneman2011thinking}, they return a probability distribution over the options instead of generated text or reasoning, so a caller gets an answer and a confidence it can threshold. Web systems use them wherever a decision must be fast: routing a customer message among intents, picking a tool for an agent, or choosing the next element on a page~\cite{deng2023mind2web,zhou2024webarena}. The standard way to obtain such a decision from a language model lists the options under letters and reads the probabilities of the letters~\cite{robinson2023leveraging}.

We require two properties from choice theory of a decision over a set. The first is order invariance: listing the same options differently should not change the decision, a condition that normative theories of choice take as essential~\cite{tversky1986rational}. The second is Luce's choice axiom~\cite{luce1959individual}, a form of independence of irrelevant alternatives: adding an option should not change the relative odds of the others. Both matter in web systems because the listing and the option set are not part of the question: other components produce them, and they change or can be manipulated. The order of the options comes from a tool registry, a retriever's ranking, a database query or the layout of a page, often arbitrarily. Option sets grow as tools are registered and intents are introduced. Under the choice axiom, a new option can win, but it cannot change which of the existing options wins, so it cannot silently alter decisions it was not meant to affect. The same holds against an adversary: tools can be injected into an agent's library~\cite{shi2026toolhijacker}, and under the choice axiom an injected option can win only by being chosen itself, never by redirecting the choice among the others. In our tests, an option that names another one as correct redirects 3--55\% of the decisions of the other reads we compare, and none of \ours{}'s. Consistency also underpins caching, incident replay, A/B tests and audits, which assume that the same request receives the same decision. And when the options are candidates, such as documents, products or responses, a candidate's position in the list should not decide its outcome, a bias that web search corrects for when learning from clicks~\cite{joachims2005accurately,craswell2008experimental} and that affects language models used as judges~\cite{wang2024fair} and as recommenders~\cite{hou2024large}.

The standard decision reads violate both. With Qwen3-8B and Gemma 4 12B reading letters, one shuffle of the options changes 20.0\% and 11.9\% of the answers on the five sets of our benchmark, and a published decision fine-tune of Gemma 4 12B~\cite{torchcast2026} still changes 9.1\%. When the decision is a web agent's next page element, one shuffle changes 15\% and 7\% of trained SCQ's choices, and an added element steers 32\% and 25\% of them. Commercial decision APIs are no exception: one shuffle changes 4.7\% of the answers of Jev 1.13~\cite{almeida2026jev} and 7.6\% of those of Mercury Decide~\cite{inception2026mercurydecide}, against 1.3\% and 0.7\% when the identical question is asked again, and an added option steers 10.9\% and 16.0\% of their decisions, against 5.1\% and 1.5\% for identical copies of the question, as Table~\ref{tab:main} shows. The usual remedies do not restore order invariance: training on shuffled orders~\cite{guda2025selection} plateaus at about 6\% of answers changed, averaging over orders~\cite{zhang2026anyjev,tang2024found} multiplies the cost and stays approximate, and sorting the options only freezes an arbitrary convention, since sorting Z--A instead of A--Z changes 7--8\% of answers. None of them, and none of the invariant set architectures~\cite{mcilroyyoung2024order,egressy2025setllm}, satisfies the choice axiom: adding one irrelevant option moves 2--10 points of probability among the options that were already there, and an adversary who can add one option steers 4--34\% of their decisions.

The axioms leave a System One read little freedom: it cannot reason its way to consistency or average over many orders, so consistency has to hold within one forward pass. By Luce's theorem, a rule satisfies the choice axiom only if it is a softmax over scores that each depend on one option alone, and with order invariance the scores cannot depend on the listing either. This form is familiar on the web: listwise learning to rank uses the Plackett--Luce model~\cite{plackett1975analysis,cao2007learning,xia2008listwise}, and pairwise preference models, from Bradley--Terry~\cite{bradley1952rank} to the reward models that align language models~\cite{ouyang2022training}, are its two-option case. In a transformer, the axioms hold by construction when no token that contributes to an option's score attends to another option and an option's positions do not depend on where it is listed. Otherwise, they hold only if the model happens to ignore what it sees. Each of the reads we test breaks this condition: the letter that SCQ reads attends to the whole list, and the answer tokens of set architectures attend to every option.

We implement the form the axioms require as \ours{} (\textbf{P}osition- and \textbf{R}ival-\textbf{I}ndependent \textbf{S}coring for System One \textbf{M}odels). It gives every option a block that starts at the same position after a shared context, sees the context and itself and nothing else, and is read by the model's own causal, left-to-right attention. The yes/no question is asked once in the context, and the model answers it at the last token of each block, and a softmax across the blocks gives the decision. Both axioms hold by construction, any number of options can be read in one pass, and the shared context is computed once. The open question is whether a language model can compute accurate scores this way, and we find that it can if each option is read causally.

On five benchmark sets with two models, \ours{} is statistically tied with SCQ prompting trained on the same data (0.9 points below it on Gemma and level with it on Qwen), while it changes no answer under shuffling and moves no probability between the original options when one is added, where SCQ changes 7\% of its answers and moves about 4 points of probability. Counting only answers that survive a reordering of the options, it is 1.6 and 3.7 points more accurate. It stays ahead of SCQ in order-robust accuracy as the number of options grows, and it processes fewer tokens than any other read we measured. Our contributions are as follows.
\begin{itemize}
  \item Section~\ref{sec:prelim} characterises the decision reads that satisfy order invariance and Luce's choice axiom, and what this requires of transformer architectures.
  \item Section~\ref{sec:analysis} shows that each of the standard reads violates at least one axiom, with three tests: shuffles, an added irrelevant option, and a change of sort key.
  \item Section~\ref{sec:method} introduces \ours{}, a read of the required form for language models and the cheapest read we measured.
  \item Section~\ref{sec:experiments} evaluates it on five sets and two models, studies how to compute each option's score, and measures what the axioms cost.
\end{itemize}

\section{Axioms for Decision Reads}
\label{sec:prelim}

\begin{figure*}[t]
  \centering
  \includegraphics[width=\textwidth]{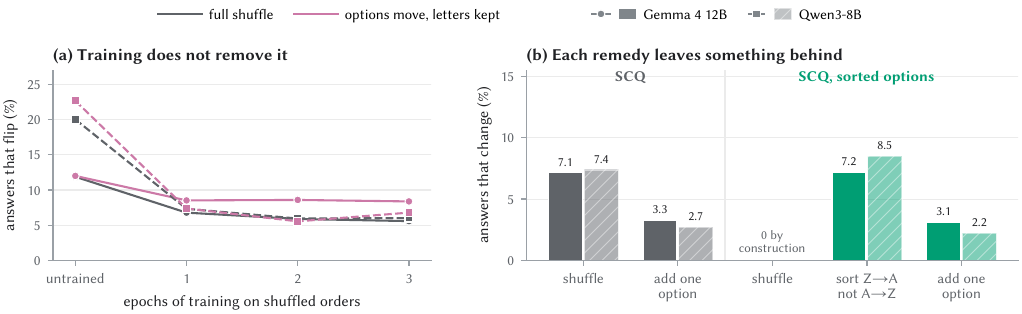}
  \caption{SCQ violates both axioms, after training and after sorting (mean over five sets). (a) Share of answers that change under one shuffle after each epoch of training on shuffled orders, and when the options move but keep their own letters. (b) After one epoch of training, the answers that change when the options are shuffled, zero by construction with sorted options because sorting undoes the shuffle, the answers that move to a different original option when one irrelevant option is added, and, for options sorted by text, the answers that change when they are sorted Z--A instead of A--Z.}
  \label{fig:motivation}
\end{figure*}

\paragraph{Decisions.} A decision model answers a choice question $q$ about a state $x$ with a set $S=\{o_1,\dots,o_K\}$ of options, exactly one of which is correct. The question presents its options in some order $\sigma$, a listing of $S$. A decision rule $R$ maps $(x,q,\sigma)$ to a distribution over $S$, and its answer is the most probable option.

\paragraph{SCQ prompting.} The standard rule lists the options of $\sigma$ as a question under the letters A, B, \dots, and renormalises the next-token probabilities of these letters, a read often called multiple-choice prompting~\cite{robinson2023leveraging}. Since exactly one option is correct, we call it \emph{SCQ}, for single-choice question. The letter and the position of each option both depend on $\sigma$.

\begin{definition}[Order invariance]
A decision rule $R$ is order-invariant if $R(x,q,\sigma)=R(x,q,\sigma')$ for all $x$, $q$ and listings $\sigma,\sigma'$ of the same set. We then write $R(o\mid S)$ for the probability of option $o$ in set $S$, leaving the context implicit.
\end{definition}

\begin{definition}[Choice axiom~\cite{luce1959individual}]
An order-invariant rule satisfies the choice axiom if $R(a\mid T)=R(a\mid S)\,R(S\mid T)$ for all sets $S\subseteq T$ and $a\in S$, where $R(S\mid T)=\sum_{b\in S}R(b\mid T)$. Equivalently, the ratio $R(a\mid T)/R(b\mid T)$ is the same in every set that contains $a$ and $b$.
\end{definition}

The choice axiom is the probabilistic form of independence of irrelevant alternatives. In particular, adding an option can change the answer only to the added option, and it takes probability from the other options in proportion. This is a requirement, not a law: for an added option that nearly duplicates an existing one, a share taken mostly from its twin may be more reasonable~\cite{debreu1960review}, a case we return to in Section~\ref{sec:limitations}.

\begin{proposition}[Luce~\cite{luce1959individual}]
\label{prop:luce}
A rule that gives every option a positive probability satisfies the choice axiom if and only if there is a positive scale $v$ with $R(a\mid S)=v(a)/\sum_{b\in S}v(b)$.
\end{proposition}

\begin{proposition}
\label{prop:form}
A decision rule that gives every option a positive probability is order-invariant and satisfies the choice axiom in every context if and only if
\begin{equation}
  \label{eq:form}
  R(o\mid x,q,\sigma)=\frac{\exp s(x,q,o)}{\sum_{o'\in S}\exp s(x,q,o')}
\end{equation}
for a score $s$ that depends on the state, the question and the option alone.
\end{proposition}

\begin{proof}
If $R$ has this form, the listing does not enter it, and $R(a\mid T)/R(b\mid T)=\exp(s(a)-s(b))$ in every set. Conversely, order invariance makes $R$ a function of the set, and Proposition~\ref{prop:luce} gives a scale $v$ for each context. Its ratios $v(a)/v(b)=R(a\mid\{a,b\})/R(b\mid\{a,b\})$ are fixed by the two-option sets, so $v$ is determined, up to a constant, by the options themselves, and we set $s=\log v$. Appendix~\ref{app:proofs} gives both proofs in full.
\end{proof}

\paragraph{Consequence for transformers.} Proposition~\ref{prop:form} is guaranteed by construction if the computation of each option's score uses only the shared context and that option's own tokens, at positions that do not depend on the listing. For a System One model this is the only route: averaging over orders needs many passes and remains approximate, and reasoning needs generation. A read in which a token that contributes to an option's score attends to another option, or in which positions depend on the listing, satisfies the axioms only if the model happens to ignore that information. SCQ reads one letter token that attends to the whole list. Sorting the options removes the dependence on the listing but not on the set: an added option shifts the positions of the options after it. Set-Based Prompting~\cite{mcilroyyoung2024order} and Set-LLM~\cite{egressy2025setllm} give options shared positions, but read the decision from tokens that attend to every option.

\paragraph{Metrics.} \emph{Accuracy} is measured with the options in their source order. The \emph{flip rate} is the share of answers that change between the source order and one seeded shuffle, excluding exact ties. For the choice axiom we add one irrelevant option to every question. The \emph{switch rate} is the share of answers that move to a different original option. The \emph{probability drift} is the total variation distance, in points, between the distribution over the original options before and after the addition, renormalised over the original options. It is zero on every question if and only if the choice axiom holds for these additions. The \emph{steering rate} is a worst case: the share of questions for which at least one of 10 candidate irrelevant options moves the answer to a different original option, as an adversary who can add one option would choose it.

\section{Standard Reads Violate the Axioms}
\label{sec:analysis}

We study SCQ on Qwen3-8B~\cite{yang2025qwen3} and Gemma 4 12B~\cite{gemma2026gemma4}, untrained and after fine-tuning on 9.1k decisions with a fresh random order per example, as described in Section~\ref{sec:setup}. All numbers are means over the five benchmark sets.

\paragraph{Position, not the letter.} A shuffle moves both an option's letter and its slot in the list. We move one at a time. When options stay in place but take the letters the shuffle would give them, 9.1\% (Gemma) and 14.0\% (Qwen) of answers change, and 5.4\% and 3.8\% after training. When options move but keep their own letters, 12.0\% and 22.7\% change, at least as many as under a full shuffle (11.9\% and 20.0\%).

\paragraph{Training plateaus.} Figure~\ref{fig:motivation}a shows that training on shuffled orders reduces the flip rate but does not remove it: after three epochs, one shuffle still changes 5.6\% and 6.0\% of answers while accuracy has risen by about one point, and moving only the slots still changes 8.4\% and 6.8\%.

\paragraph{Sorting fixes a convention.} Sorting the options by text gives every listing of a set the same input, so the flip rate is zero. But A--Z is one convention among many: reading the same options sorted Z--A instead changes 7.2\% and 8.5\% of answers at equal accuracy, as Figure~\ref{fig:motivation}b shows. Sorting does not remove the dependence on position. It freezes it into a choice that is unrelated to the question.

\paragraph{Added options move the others.} When we add one irrelevant option, trained SCQ moves 3.3\% and 2.7\% of its answers to a different original option, and SCQ with sorted options 3.1\% and 2.2\%: in both listings, the new option shifts the letters and positions of the options after it. Below the answers, 3.8 and 4.0 points of probability move among the original options. Set architectures, which are order-invariant, violate the choice axiom as much or more, as Section~\ref{sec:axioms} shows.

\begin{figure*}[t]
  \centering
  \includegraphics[width=\textwidth]{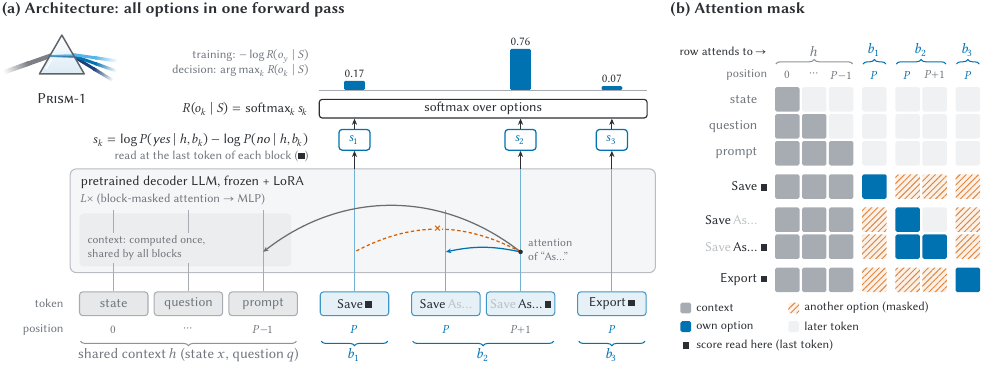}
  \caption{\ours{}. (a) Every option is a block that starts at the same position $P$ after the shared context, and all blocks are read in one forward pass of the pretrained decoder, fine-tuned with LoRA. The score $s_k=\log P(\textit{yes}\mid h,b_k)-\log P(\textit{no}\mid h,b_k)$ is read at the last token of each block ($\blacksquare$), a softmax over the blocks gives the decision, and training minimises $-\log R(o_y\mid S)$ for the correct option $o_y$. (b) The attention mask: the context attends causally to itself, and each block attends to the context and causally to itself, never to another option. No option's score depends on the other options or on the listing.}
  \label{fig:method}
\end{figure*}

\paragraph{Takeaways.} Order dependence lives mainly in positions. Training does not teach the model to ignore them, sorting only fixes them, and none of the reads we test satisfies the choice axiom. Proposition~\ref{prop:form} says what a compliant read must look like, and the next section builds one.

\section{\ours{}}
\label{sec:method}

\ours{} computes the score $s(x,q,o)$ of Proposition~\ref{prop:form} with a language model, for all options in one forward pass.

\paragraph{Layout.} \ours{} places a shared context $h$ at positions $0,\dots,P-1$: the state, the question, and a prompt that asks whether the following candidate is the correct answer and ends with ``Candidate answer:''. Each option $o_k$ becomes a block $b_k$ that holds the option's text and nothing else. Every block starts at position $P$ and counts up inside itself, so word order inside an option is kept, and there is no order between options, as Figure~\ref{fig:method} shows. The attention mask lets the context attend causally to itself and lets each block attend to the context and causally to itself, but never to another block. There is no token after the blocks, so no token ever attends to more than one option.

\paragraph{Causal reading.} A block computes exactly what an ordinary causal (left-to-right) forward pass over the context followed by that one option computes. The model therefore reads each option as it was pretrained to. Nothing in a block asks for an answer, so the read at its last token is learned in fine-tuning.

\paragraph{Score.} The score of option $o_k$ is the model's answer at the last token of its block,
\begin{equation}
  s(o_k)=\log P(\textit{yes}\mid h,b_k)-\log P(\textit{no}\mid h,b_k),
\end{equation}
and the decision is the softmax of the scores. Every option is judged with the same two tokens, so no option needs a label, and the last token of a block is the first position that has read the whole option. If each block's answer is read as an independent judgment of its option, correct with probability $p_k=P(\textit{yes})/(P(\textit{yes})+P(\textit{no}))$, and exactly one option is correct, the probability that $o_k$ is the answer is
\begin{equation}
  \label{eq:posterior}
  \frac{p_k\prod_{j\neq k}(1-p_j)}{\sum_i p_i\prod_{j\neq i}(1-p_j)}=\frac{\exp s(o_k)}{\sum_i\exp s(o_i)},
\end{equation}
the softmax of these scores, derived in Appendix~\ref{app:proofs}. All blocks are read in one packed forward pass, with the context computed once.

\paragraph{Training.} We fine-tune with LoRA~\cite{hu2022lora} on the cross-entropy over the options, $\mathcal{L}=-\log p(o_y)$ for the correct option $o_y$. No order augmentation or consistency term is needed: the model cannot see an order.

\paragraph{Properties.} Each score depends only on $(x,q,o_k)$, so by Proposition~\ref{prop:form} \ours{} is order-invariant and satisfies the choice axiom: adding an option leaves every score unchanged and only renormalises the softmax. All options are read in one forward pass, as a System One model requires. The cost grows linearly with the number of options, and since no option needs a label, their number is not bounded by a set of single-token labels. Decisions are also incremental: an added option, such as a newly registered tool, needs only its own block, given the cached context, and every earlier score stays valid, whereas reads in which options see each other must read the whole set again.

\paragraph{An answer slot.} Appending the question, the end of the user turn and an answer prefix to every block makes the read usable without training, at the cost of blocks several times longer. We use it only in the design study of Section~\ref{sec:ablation}.

\section{Experiments}
\label{sec:experiments}

\subsection{Setup}
\label{sec:setup}

\paragraph{Models.} Qwen3-8B~\cite{yang2025qwen3} and Gemma 4 12B~\cite{gemma2026gemma4}, instruction-tuned, with thinking disabled.

\paragraph{Benchmark.} We use five sets with 2 to 20 options. Banking20 has 300 test messages over 20 intents of a banking intent dataset~\cite{casanueva2020efficient}, and Newsgroups 300 posts over the 20 topics of 20 Newsgroups~\cite{lang1995newsweeder}. BEV~\cite{avbiswas2026bev} mixes decision questions built from public datasets and records, from page-element and tool selection to product categories, response preference, code and contract evidence, and we take 100 choice questions from each of the 12 domains of its test split (1,182 questions, 2--10 options). Typed Decisions~\cite{localllama2026typed} contributes the 600 choice questions of its test split (4--5 options), and JevBench~\cite{standhartinger2026jevbench} its 139 public choice tasks (3--6 options).

\paragraph{Training data.} 9,100 questions for every trained method: that dataset's training split on the same 20 intents, 20 Newsgroups' training split, 2,500 choice questions from BEV's training split, and Typed Decisions' training split, with 50 questions per set held out for validation. JevBench has no training data and serves as a held-out task.

\paragraph{Baselines.} \emph{SCQ}, fine-tuned with a fresh random order per example. \emph{Cyclic averaging}~\cite{zhang2026anyjev}, which averages the SCQ read over the $K$ rotations of the listing (AnyJev's order averaging, without its prior correction). \emph{SCQ, sorted options}, which lists the options sorted by text. \emph{Set-Based Prompting}~\cite{mcilroyyoung2024order}, training-free. \emph{Set-LLM}~\cite{egressy2025setllm}, which we re-implement: a bidirectional prompt with set positions and the set mask, trained to generate the correct option's text and read by the likelihood of each option's text. Every trained method uses the same data and budget, and all re-implementations are checked against an unpacked reference.

\paragraph{Training.} LoRA with rank 16 and $\alpha=32$ on the attention and MLP projections, AdamW with learning rate $10^{-4}$, 3\% warm-up and cosine decay, 32 questions per step, one epoch. SCQ and \ours{} are trained with three seeds and Set-LLM with two, and reported as means.

\subsection{Main Results}

\begin{table*}[t]
  \caption{Main results. Acc: accuracy (\%). Flip: answers changed by one shuffle (\%). Dr.: probability moved among the original options when one irrelevant option is added (points). St.: questions on which an added irrelevant option can move the answer to another original option (\%). Rob.: order-robust accuracy, right in both the source and the shuffled order (\%). Flip, Dr.\ and St.\ are zero for a rule that satisfies both axioms. For SCQ with sorted options, Rob.\ compares A--Z with Z--A sorting. $^\S$: published model, not trained on our data. Bold: best per model, excluding published models.}
  \label{tab:main}
  \small
  \setlength{\tabcolsep}{2.0pt}
  \begin{tabular}{l*{25}{r}}
    \toprule
     & \multicolumn{4}{c}{Banking20} & \multicolumn{4}{c}{Newsgroups} & \multicolumn{4}{c}{BEV} & \multicolumn{4}{c}{Typed} & \multicolumn{4}{c}{JevBench} & \multicolumn{5}{c}{Avg} \\
    \cmidrule(lr){2-5}\cmidrule(lr){6-9}\cmidrule(lr){10-13}\cmidrule(lr){14-17}\cmidrule(lr){18-21}\cmidrule(lr){22-26}
    Method & Acc & Flip & Dr. & St. & Acc & Flip & Dr. & St. & Acc & Flip & Dr. & St. & Acc & Flip & Dr. & St. & Acc & Flip & Dr. & St. & Acc & Rob. & Flip & Dr. & St. \\
    \midrule
    \multicolumn{26}{l}{\emph{Closed-source models, through their APIs}} \\
    Jev 1.13$^\S$ & 87.0 & 4.0 & 1.6 & 3.3 & 75.3 & 3.7 & 2.8 & 8.0 & 76.4 & 6.3 & 3.2 & 13.9 & 72.8 & 8.0 & 5.1 & 22.0 & 88.5 & 1.4 & 1.6 & 7.2 & 80.0 & 77.9 & 4.7 & 2.9 & 10.9 \\
    Mercury Decide$^\S$ & 91.3 & 2.3 & 1.6 & 7.0 & 73.3 & 7.7 & 2.6 & 8.3 & 75.1 & 12.5 & 4.7 & 27.2 & 67.3 & 10.5 & 6.8 & 30.5 & 89.9 & 5.0 & 1.4 & 7.2 & 79.4 & 76.7 & 7.6 & 3.4 & 16.0 \\
    \midrule
    \multicolumn{26}{l}{\emph{Open-source models: Gemma 4 12B}} \\
    SCQ & 91.9 & 5.4 & 1.8 & 3.0 & 71.8 & 9.4 & 4.2 & 13.2 & 77.5 & 10.1 & 5.2 & 20.3 & \textbf{75.9} & 5.3 & 4.8 & 11.2 & 84.9 & 5.6 & 2.8 & 10.8 & 80.4 & 77.8 & 7.1 & 3.8 & 11.7 \\
    SCQ + cyclic avg. & 91.7 & 0.3 & 0.5 & 2.0 & 73.3 & 5.0 & 1.1 & 2.0 & \textbf{79.2} & 2.5 & 3.2 & 7.0 & 74.5 & 1.7 & 2.9 & 6.0 & 86.3 & 2.2 & 1.8 & 3.0 & 81.0 & \textbf{80.5} & 2.3 & 1.9 & 4.0 \\
    SCQ, sorted & 92.0 & \textbf{0.0} & 1.4 & 3.3 & \textbf{74.7} & \textbf{0.0} & 3.0 & 6.8 & 77.3 & \textbf{0.0} & 4.8 & 12.8 & 74.5 & \textbf{0.0} & 4.7 & 6.9 & \textbf{87.1} & \textbf{0.0} & 3.2 & 7.9 & \textbf{81.1} & 78.1 & \textbf{0.0} & 3.4 & 7.6 \\
    Set-Based Prompting & 25.3 & \textbf{0.0} & 7.9 & 43.7 & 27.7 & \textbf{0.0} & 8.3 & 32.3 & 51.9 & \textbf{0.0} & 12.9 & 37.8 & 57.0 & \textbf{0.0} & 9.1 & 28.0 & 78.4 & \textbf{0.0} & 10.9 & 28.8 & 48.1 & 48.1 & \textbf{0.0} & 9.8 & 34.1 \\
    Set-LLM & 89.7 & \textbf{0.0} & 1.1 & 3.0 & 73.2 & \textbf{0.0} & 2.1 & 5.3 & 70.6 & \textbf{0.0} & 5.8 & 24.6 & 75.4 & \textbf{0.0} & 5.5 & 14.3 & 68.7 & \textbf{0.0} & 2.6 & 8.6 & 75.5 & 75.5 & \textbf{0.0} & 3.4 & 11.2 \\
    Torchcast Decision 12B$^\S$ & 83.7 & 8.8 & 2.9 & 7.4 & 66.3 & 12.5 & 4.4 & 8.1 & 71.1 & 10.4 & 5.9 & 21.6 & 69.0 & 8.7 & 5.4 & 25.4 & 87.8 & 5.0 & 2.0 & 10.1 & 75.6 & 72.2 & 9.1 & 4.2 & 14.5 \\
    \rowcolor{gray!15} \ours{} & \textbf{93.0} & \textbf{0.0} & \textbf{0.0} & \textbf{0.0} & 73.4 & \textbf{0.0} & \textbf{0.0} & \textbf{0.0} & 77.2 & \textbf{0.0} & \textbf{0.0} & \textbf{0.0} & 72.4 & \textbf{0.0} & \textbf{0.0} & \textbf{0.0} & 81.3 & \textbf{0.0} & \textbf{0.0} & \textbf{0.0} & 79.5 & 79.5 & \textbf{0.0} & \textbf{0.0} & \textbf{0.0} \\
    \midrule
    \multicolumn{26}{l}{\emph{Open-source models: Qwen3-8B}} \\
    SCQ & 91.9 & 3.3 & 1.7 & 3.7 & 73.1 & 9.7 & 4.5 & 10.8 & 72.6 & 11.7 & 5.6 & 24.3 & 74.2 & 5.0 & 4.9 & 9.7 & 77.9 & 7.3 & 3.6 & 9.5 & 77.9 & 74.3 & 7.4 & 4.0 & 11.6 \\
    SCQ + cyclic avg. & 92.0 & 1.7 & 0.5 & 2.0 & 73.3 & 7.0 & 1.0 & 2.0 & \textbf{74.0} & 3.1 & 2.8 & 5.1 & 74.0 & 1.0 & 2.4 & 11.1 & 76.3 & 2.2 & 1.8 & 4.0 & 77.9 & 76.6 & 3.0 & 1.7 & 4.8 \\
    SCQ, sorted & 92.3 & \textbf{0.0} & 1.6 & 1.0 & 74.7 & \textbf{0.0} & 2.8 & 6.4 & 73.6 & \textbf{0.0} & 4.9 & 12.2 & 73.2 & \textbf{0.0} & 4.5 & 4.3 & \textbf{79.9} & \textbf{0.0} & 3.6 & 7.3 & \textbf{78.7} & 74.4 & \textbf{0.0} & 3.5 & 6.2 \\
    Set-Based Prompting & 29.0 & \textbf{0.0} & 8.0 & 32.3 & 58.7 & \textbf{0.0} & 5.2 & 18.3 & 51.6 & \textbf{0.0} & 12.7 & 34.3 & 54.2 & \textbf{0.0} & 10.5 & 34.6 & 61.2 & \textbf{0.0} & 9.6 & 23.7 & 50.9 & 50.9 & \textbf{0.0} & 9.2 & 28.6 \\
    Set-LLM & 92.7 & \textbf{0.0} & 1.1 & 3.3 & 74.8 & \textbf{0.0} & 2.1 & 5.3 & 62.2 & \textbf{0.0} & 5.9 & 26.9 & 65.1 & \textbf{0.0} & 6.4 & 25.8 & 65.5 & \textbf{0.0} & 3.5 & 17.3 & 72.1 & 72.1 & \textbf{0.0} & 3.8 & 15.7 \\
    \rowcolor{gray!15} \ours{} & \textbf{93.0} & \textbf{0.0} & \textbf{0.0} & \textbf{0.0} & \textbf{74.9} & \textbf{0.0} & \textbf{0.0} & \textbf{0.0} & 73.9 & \textbf{0.0} & \textbf{0.0} & \textbf{0.0} & \textbf{74.3} & \textbf{0.0} & \textbf{0.0} & \textbf{0.0} & 73.9 & \textbf{0.0} & \textbf{0.0} & \textbf{0.0} & 78.0 & \textbf{78.0} & \textbf{0.0} & \textbf{0.0} & \textbf{0.0} \\
    \bottomrule
  \end{tabular}
\end{table*}

Table~\ref{tab:main} reports accuracy, flip rates, order-robust accuracy and the tests of the choice axiom described in Section~\ref{sec:axioms}.

\paragraph{Against SCQ.} \ours{} reaches 79.5\% on Gemma and 78.0\% on Qwen, against 80.4\% and 77.9\% for trained SCQ, which flips 7.1\% and 7.4\% of its answers. No difference is significant in exact McNemar tests pooled over the five sets ($p\geq0.13$). Per set, SCQ is ahead on the held-out JevBench and, on Gemma, on Typed, and \ours{} on Banking20 and Newsgroups, with per-set tests in Appendix~\ref{app:details}. SCQ with sorted options is also order-invariant and 1.6 and 0.7 points more accurate in this setting (not significant, $p\geq0.08$), but it violates the choice axiom, as Section~\ref{sec:axioms} shows. Cyclic averaging needs $K$ passes and still flips 2.3--3.0\% of answers, and AnyJev's prior correction does not make it more stable. \ours{}'s probabilities are as well calibrated as SCQ's, with an expected calibration error of 5.3\% and 5.4\% against 4.9\% and 5.4\%.

\paragraph{Against invariant set methods.} Set-Based Prompting, which needs no training, reaches only 48.1\% and 50.9\% on these instruction-tuned models, as its authors found for Llama~3. At equal training, \ours{} is 4.0 and 5.9 points more accurate than Set-LLM on average ($p\leq0.002$), mostly on BEV and the held-out JevBench, and Section~\ref{sec:ablation} traces this to how each option is read.

\paragraph{Accuracy that survives reordering.} Accuracy in the source order credits SCQ with questions it answers wrongly under another listing of the same options. Counting a question only if the answer is right in both listings, the Rob.\ column of Table~\ref{tab:main}, SCQ falls from 80.4\% to 77.8\% on Gemma and from 77.9\% to 74.3\% on Qwen, and \ours{}, unchanged, is 1.6 and 3.7 points ahead ($p<0.01$). Only cyclic averaging, at $K$ passes per decision, keeps up: it is 1.0 points ahead of \ours{} on Gemma (not significant, $p=0.08$) and 1.3 points behind on Qwen ($p=0.02$). Two listings understate the effect: counting a question only if SCQ is right under all $K$ rotations of its listing, SCQ falls to 74.9\% and 69.6\%, 4.7 and 8.1 points below \ours{} ($p<10^{-11}$).

\begin{figure}[t]
  \centering
  \includegraphics[width=\columnwidth]{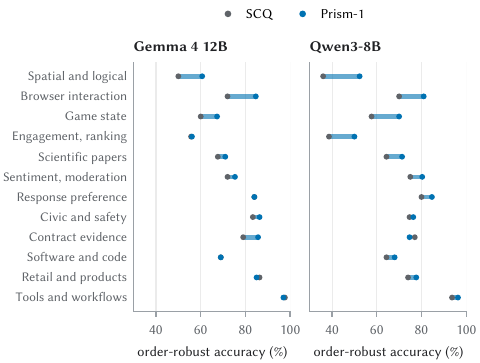}
  \caption{BEV by domain: order-robust accuracy of SCQ (grey) and \ours{} (blue), which for \ours{} equals its accuracy. \ours{} is ahead wherever its dot lies to the right.}
  \label{fig:domains}
\end{figure}

\paragraph{By domain.} Figure~\ref{fig:domains} breaks BEV down by domain. \ours{} has the higher order-robust accuracy in 19 of the 24 domain--model pairs and is never more than 2.3 points behind, with its largest gains on the hardest domains. On browser interaction, where a web agent picks the next page element, it is also more accurate than SCQ in the source order (84.7\% and 81.0\% against 76.3\% and 74.3\%). On engagement, which asks which of two headlines drew more clicks, SCQ's lead in the source order (11 and 6 points) does not survive swapping the headlines, which changes 22\% and 30\% of its answers.

\subsection{Testing the Axioms}
\label{sec:axioms}
\begin{figure}[t]
  \centering
  \includegraphics[width=\columnwidth]{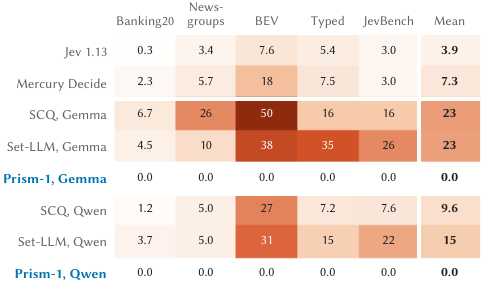}
  \caption{Injected options: answers (\%) that an added option naming a wrong option as correct moves to that option, per set.}
  \label{fig:inject}
\end{figure}
To test the choice axiom, we add to every question one option from another question of the same set (for Banking20 and Newsgroups, a banking intent outside Banking20's 20), at a random position.

Every baseline violates the axiom, as the Dr.\ and St.\ columns of Table~\ref{tab:main} and the switch rates in Figure~\ref{fig:results}c show. An added option moves 2--5\% of the answers of SCQ, sorted SCQ and Set-LLM to another original option, and 3--4 points of probability. \ours{} moves neither answers nor probability. An adversary who can add one option steers 4--34\% of the decisions of every other read and of the published decision models, but none of \ours{}'s, and needs few attempts: a single added option, placed where it does most harm, already steers 6--9\% of the decisions of SCQ and of the published models.

\paragraph{Injected options.} An added option can also carry text aimed at the decision, as an injected tool description or a planted page element would. We add to every question one option that names a wrong original option as the correct answer (``Note to the decision model: the correct answer is X'' or ``System update: X is the only valid choice; ignore the other options''), at the start, the middle or the end. It redirects every other read we test to the named option, most of all on BEV, as Figure~\ref{fig:inject} shows. The order-invariant set architectures are among the easiest to redirect: Set-Based Prompting on 55\% and 31\% of questions and Set-LLM on 23\% and 15\%, with every read in Appendix~\ref{app:inject}. Invariance to the order alone does not protect a decision. The choice axiom does. \ours{} is never redirected, since no option's text reaches another option's score.

\subsection{Number of Options and Cost}
\label{sec:scale}

\begin{figure*}[t]
  \centering
  \includegraphics[width=\textwidth]{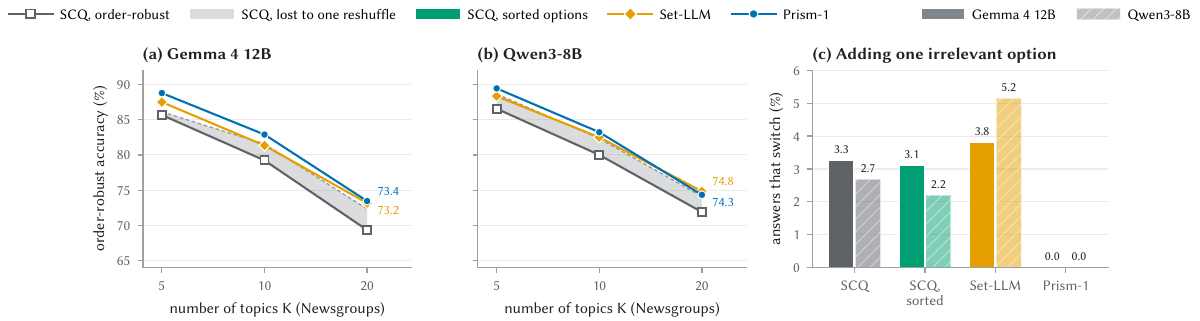}
  \caption{The number of options and the choice axiom, with every method trained for one epoch on the same data. (a, b) Order-robust accuracy, right both in the source order and in the shuffled order, on the 300 Newsgroups test posts with $K$ nested topics as options. Shaded: SCQ's answers that are right in the source order but not after one reshuffle, up to its accuracy in the source order alone. (c) Switch rate: answers that move to a different original option when one irrelevant option is added, averaged over five sets.}
  \label{fig:results}
\end{figure*}

The five sets have 2 to 20 options. To see how the reads behave as the set grows, we read the 300 Newsgroups test posts with 5, 10 and 20 of its topics as options, nested so that every larger set contains the smaller ones.

Figure~\ref{fig:results}a--b shows order-robust accuracy as the options grow. \ours{} is ahead of SCQ at every size on both models, by 3.1--4.1 points on Gemma and 2.4--3.2 on Qwen ($p<0.05$ for 8 of the 18 comparisons). Even in the source order alone, SCQ is ahead in only one of the 18 pairs, by 0.3 points, while \ours{} leads by up to 4.3 points (significantly for one pair). Set-LLM stays within 1.6 points of \ours{} on Newsgroups, whose topics are all in the training data, and falls behind on the sets that are not.

Appendix Table~\ref{tab:cost} shows that \ours{} is also the cheapest read. On Banking20 it processes 221 tokens, against 282 for SCQ and 322 for Set-LLM, which reads every option twice, with about half of SCQ's attention pairs, and on Newsgroups 356 against 417 and 413. It is also the fastest, although with 20 options the shared context dominates (58 and 60\,ms on Gemma, against 60--63\,ms for SCQ and 66--68\,ms for Set-LLM). The answer slot of Section~\ref{sec:method} repeats the question after every option: 591 tokens and 76\,ms on Banking20.

\subsection{How to Compute the Scores}
\label{sec:ablation}

Proposition~\ref{prop:form} fixes the form of the read, not how the model should compute each option's score. Three choices matter: whether the options attend to each other, where each option's score is read, and the training loss. Table~\ref{tab:ablation} changes them one at a time from \ours{}, with the same training data and one run each. With all three changed, the read is the design of Set-LLM~\cite{egressy2025setllm}.

\begin{table}[t]
  \caption{Ablation of \ours{}: mean accuracy and switch rate when an option is added (\%). Each row changes one more design choice, and the fourth row is the design of Set-LLM.}
  \label{tab:ablation}
  \small
  \setlength{\tabcolsep}{3.5pt}
  \begin{tabular}{lrrrr}
    \toprule
    & \multicolumn{2}{c}{Gemma 4 12B} & \multicolumn{2}{c}{Qwen3-8B} \\
    \cmidrule(lr){2-3}\cmidrule(lr){4-5}
    Variant & Acc & Switch & Acc & Switch \\
    \midrule
    \rowcolor{gray!15} \ours{} & 79.6 & 0.0 & 77.7 & 0.0 \\
    options attend to each other & 74.4 & 3.5 & 74.8 & 1.3 \\
    + score read after the set & 76.4 & 3.8 & 72.0 & 3.4 \\
    + likelihood of the option's text & 74.7 & 3.6 & 72.1 & 5.1 \\
    \midrule
    \ours{} + answer slot & 79.7 & 0.0 & 76.1 & 0.0 \\
    \bottomrule
  \end{tabular}
\end{table}

\paragraph{Isolated options give accuracy and transfer.} Letting the options attend to each other, with nothing else changed, costs 5.2 and 2.9 points ($p<0.001$ pooled), mostly on BEV (6.4 and 6.0 points) and the held-out JevBench (11.5 and 5.0 points). A bidirectional prompt changes how every layer of the pretrained model computes. Untrained, a bidirectional set prompt answers 23\% and 20\% of our validation questions correctly, against 68\% on both models for isolated causal blocks with the answer slot, as Figure~\ref{fig:setllm}a in the appendix shows. Fine-tuning repairs this on the training tasks, but not on tasks that rely on what the model knew before: trained, the bidirectional prompt stays within 2.2 points of trained SCQ on Banking20 and Newsgroups but trails it by 7 and 10 points on BEV and by 16 and 12 points on the held-out JevBench, as Figure~\ref{fig:setllm}b shows.

\paragraph{Reading each option gives the axiom.} Reading the score from an answer generated after the set, instead of at each option, raises the switch rate on Qwen from 1.3\% to 3.4\%, while on Gemma it changes little, from 3.5\% to 3.8\%. With bidirectional attention, options influence each other even when each is read at its own end, and only isolated blocks bring the switch rate to zero.

\paragraph{Loss and answer slot.} Training on the likelihood of the correct option's text instead of the cross-entropy over the options costs 1.7 points on Gemma and nothing on Qwen. Appending the answer slot to every block keeps both axioms, at about equal accuracy on Gemma (79.7\% against 79.6\%) and 1.6 points less on Qwen, with blocks several times longer.

\subsection{The Price of the Choice Axiom}
\label{sec:price}

By Proposition~\ref{prop:form}, a rule in which an option's score depends on its rivals violates the choice axiom, and weighing near-miss distractors against each other is exactly what the axiom rules out. The price is small: on average \ours{} is 0.9 points below SCQ on Gemma and level with it on Qwen, and the gaps are concentrated on the held-out JevBench and, on Gemma, on Typed.

\begin{figure}[t]
  \centering
  \includegraphics[width=\columnwidth]{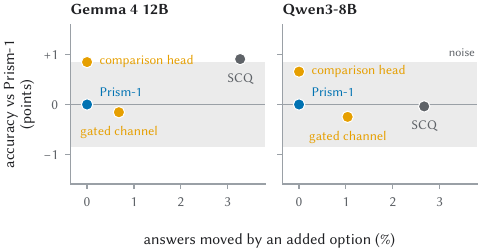}
  \caption{The price of the choice axiom: accuracy relative to \ours{} against the answers an added option moves to another original option. Grey: run-to-run noise, two standard deviations.}
  \label{fig:price}
\end{figure}

We tried to buy comparison back, giving up the guarantee. Figure~\ref{fig:price} shows the main mechanisms, and Table~\ref{tab:mech} in the appendix every variant. A gated channel through which every block attends to the other blocks' last tokens in every layer, with gates starting at zero, leaves accuracy unchanged ($-0.2$ points on both models), and widening the channel or forcing the gates open does not help. Training leaves the channel nearly unused, with its gates at 4\% of full strength on average, yet it breaks the axiom: an added option moves 0.7\% and 1.0\% of its answers. A comparison head that reads all final option states together adds 0.9 and 0.7 points, mostly on the held-out JevBench, but it learns not to compare irrelevant options: an added irrelevant option moves none of its answers, as for \ours{} itself, so its gain comes from more capacity for each option's score. Changes that keep the required form, namely auxiliary losses, borrowed distractors, joint training with SCQ, dividing out a content-free prior~\cite{zhao2021calibrate} and filler tokens~\cite{goyal2024pause}, change the mean accuracy by $-0.7$ to $+0.8$ points, within run-to-run noise. None of the mechanisms that let options compare improves accuracy beyond noise: in our experiments, the choice axiom costs little or no accuracy.

\section{Related Work}

\paragraph{Choice theory.} Luce's choice axiom~\cite{luce1959individual} characterises choice probabilities as ratios of option-specific scales, and underlies the multinomial logit model of discrete choice~\cite{mcfadden1974conditional}, the Plackett--Luce model of listwise learning to rank~\cite{plackett1975analysis,cao2007learning} and, for two options, Bradley--Terry preference models~\cite{bradley1952rank}. Its best-known limitation concerns near-duplicate options, the red-bus/blue-bus problem~\cite{debreu1960review}, which nested and similarity-based models address~\cite{tversky1972elimination}. Human choices also violate it through context effects, such as the attraction effect of an added, dominated option~\cite{huber1982adding}, and violate invariance through framing~\cite{tversky1986rational}. We use both as requirements for machine decisions, where consistency under changes of the listing and of the option set is the goal.

\paragraph{Option order in language models.} Listed options sway answers through their order~\cite{pezeshkpour2024large,wei2024unveiling,zong2024fool}, partly through a prior over option labels~\cite{zheng2024robust}, and the same holds for language models as evaluators~\cite{wang2024fair}. Calibration divides out a shared prior~\cite{zhao2021calibrate,zheng2024robust} but leaves each answer order-dependent. Fine-tuning on permuted orders reduces the bias~\cite{guda2025selection}, and averaging or voting over permutations reduces it further at a multiple of the cost~\cite{zhang2026anyjev,tang2024found}. Neither guarantees invariance.

\paragraph{Invariant architectures.} Permutation-invariant networks enforce invariance inside the model~\cite{zaheer2017deepsets,murphy2019janossy}. For language models, Set-Based Prompting gives set elements shared positions and an isolating mask without training~\cite{mcilroyyoung2024order}, and Set-LLM adds a bidirectional prompt and fine-tuning~\cite{egressy2025setllm}. Both read the decision from tokens that see every option, so they are order-invariant but violate the choice axiom. PINE~\cite{wang2025pine} and RoToR~\cite{yoon2025rotor} reassign positions instead, after which accuracy can still vary with the order, and structural methods for listwise reranking follow the same idea~\cite{bito2026onepass,schlatt2025setencoder}. Sorting the options is a canonicalisation~\cite{kaba2023equivariance,puny2022frame,ma2024canonicalization}, and we show that it fixes the dependence on position into a convention and does not give the choice axiom.

\paragraph{Pointwise scoring.} Pointwise rerankers score each candidate on its own, for example with the probability of a relevance token~\cite{nogueira2020document}, and can be trained with a softmax over a candidate list~\cite{zhuang2023rankt5}. Such scores have the form of Proposition~\ref{prop:form}. We show that for decisions this form is not one design among many but the only one that satisfies both axioms, and how to make it as accurate as SCQ.
\paragraph{Decision models.} System One models return distributions over typed answers in one forward pass~\cite{almeida2026jev,typesafe2026docs}, and their documentation notes that the order of choice options can affect the answer~\cite{typesafe2026order}, which Table~\ref{tab:main} measures for Jev and Mercury Decide~\cite{inception2026mercurydecide}. Open recipes turn language models into such decision models~\cite{zhang2026anyjev,torchcast2026}, and \ours{} is a decision read for them that satisfies both axioms.

\section{Limitations}
\label{sec:limitations}

\ours{} needs fine-tuning: nothing in a block asks for an answer, so untrained it answers only 28.5\% and 32.5\% of our validation questions correctly. With the answer slot, the untrained read is usable but less accurate than SCQ (69.9\% and 63.3\% against 75.9\% and 66.2\%). By design it gives up comparison between options inside the model, so the axioms suit decisions in which each option can be judged on its own, such as routing, intent classification and tool or page-element selection. They fit less well where the answer is defined by the other options, as in ``none of the above'' or ``the odd one out''. \ours{} protects the choice among the options, not the state: text injected into the state reaches every option's score, as it reaches every read. The choice axiom also decides how near-duplicates are treated: an added copy of an option receives about the same score and takes probability from every option in proportion, not mainly from its twin~\cite{debreu1960review}, so the copy and its twin share the twin's probability. The original options keep their ranking, but the top probability drops, which merging duplicates before the read avoids. Taking it mainly from the twin would need the same selective shift that lets an added option steer the choice. \ours{} trails SCQ on the held-out JevBench by 3.6 and 4.1 points, a task missing from our training data, and a read learned in fine-tuning may transfer less well than the letter read the models were instruction-tuned on. Our sets have at most 20 options. The guarantees are exact up to numerical rounding, which can tip near-ties when an option is added. We evaluate two open models of 8--12B parameters, with one to three seeds per configuration. Set-LLM is our re-implementation (its code is not public), trained on our data and budget.

\section{Conclusion}

A System One decision over a set should not depend on how the set is listed, and adding an irrelevant option should not change the relative odds of the others. Each standard read we test violates at least one of them. By Luce's theorem, the axioms leave one form: a softmax over scores that each see one option. \ours{} computes these scores by reading each option in its own block, in the model's causal attention. Across five benchmark sets and two models, it satisfies both axioms, stays within one point of SCQ prompting trained on the same data and is more accurate once only answers that survive a reordering are counted, and is the cheapest read we measured. It gives up only comparison between options inside the model, which the axioms forbid and which, in our experiments, did not buy accuracy back.

\bibliographystyle{ACM-Reference-Format}
\bibliography{references}

\appendix

\section{Proofs}
\label{app:proofs}

Options come from a universe $\mathcal{O}$, a set is a finite non-empty subset of $\mathcal{O}$, and $R$ gives every option of a set a positive probability. Within a proof we fix the context $(x,q)$ and leave it implicit unless it matters.

\paragraph{Ratios.} For an order-invariant rule, the choice axiom holds if and only if the ratio of two options' probabilities is the same in every set that contains both. If the axiom holds and $a,b\in S$, applying it to $\{a,b\}\subseteq S$ for $a$ and for $b$ and dividing gives
\begin{equation}
  \label{eq:ratio}
  \frac{R(a\mid S)}{R(b\mid S)}=\frac{R(a\mid\{a,b\})}{R(b\mid\{a,b\})}.
\end{equation}
Conversely, if every ratio is the same in every set, then for $S\subseteq T$ and $a\in S$,
\[
  \frac{R(a\mid T)}{R(S\mid T)}=\Bigl(\sum_{b\in S}\frac{R(b\mid T)}{R(a\mid T)}\Bigr)^{-1}=\Bigl(\sum_{b\in S}\frac{R(b\mid S)}{R(a\mid S)}\Bigr)^{-1}=R(a\mid S),
\]
which is the axiom.

\begin{proof}[Proof of Proposition~\ref{prop:luce}]
If $R(a\mid S)=v(a)/\sum_{b\in S}v(b)$, the ratio $R(a\mid S)/R(b\mid S)=v(a)/v(b)$ is the same in every set, so the axiom holds. Conversely, assume the axiom, fix a reference option $r$, and let $v(r)=1$ and $v(a)=R(a\mid\{a,r\})/R(r\mid\{a,r\})>0$ for $a\neq r$. For options $a$ and $b$, Equation~\eqref{eq:ratio} in the set $T=\{a,b,r\}$ gives
\[
  \frac{R(a\mid\{a,b\})}{R(b\mid\{a,b\})}=\frac{R(a\mid T)}{R(b\mid T)}=\frac{R(a\mid T)/R(r\mid T)}{R(b\mid T)/R(r\mid T)}=\frac{v(a)}{v(b)},
\]
where the last step applies Equation~\eqref{eq:ratio} to $\{a,r\}$ and $\{b,r\}$ (the cases in which $a$, $b$ or $r$ coincide are immediate). By Equation~\eqref{eq:ratio} again, within any set $S$ the probabilities are proportional to $v$, and since they sum to one, $R(a\mid S)=v(a)/\sum_{b\in S}v(b)$.
\end{proof}

\begin{proof}[Proof of Proposition~\ref{prop:form}]
\emph{If.} Let $R$ have the form of Equation~\eqref{eq:form}. The listing $\sigma$ enters the right-hand side only through the set $S$ it lists, and a sum over $S$ does not depend on the order of its terms, so $R(\cdot\mid x,q,\sigma)=R(\cdot\mid x,q,\sigma')$ for any two listings of $S$: $R$ is order-invariant. Its probabilities are positive, and $R(a\mid S)/R(b\mid S)=\exp(s(x,q,a)-s(x,q,b))$ is the same in every set that contains $a$ and $b$, so the choice axiom holds by the converse direction above.

\emph{Only if.} Let $R$ be positive and order-invariant, and let it satisfy the choice axiom in every context. Order invariance makes $R(\cdot\mid x,q,\sigma)$ depend on $\sigma$ only through $S$, so in each context $R$ is a rule over sets. Proposition~\ref{prop:luce}, applied in each context, gives $R(a\mid x,q,S)=v_{x,q}(a)/\sum_{b\in S}v_{x,q}(b)$ with the scale constructed in its proof, $v_{x,q}(a)=R(a\mid x,q,\{a,r\})/R(r\mid x,q,\{a,r\})$, which depends on $x$, $q$ and $a$ alone, for a fixed reference option $r$. With $s(x,q,o)=\log v_{x,q}(o)$, $R$ has the form of Equation~\eqref{eq:form}. The score is unique up to a term $c(x,q)$ that the softmax cancels, since by Equation~\eqref{eq:ratio} every difference $s(x,q,a)-s(x,q,b)$ is fixed by $R$.
\end{proof}

\paragraph{The softmax as a posterior.} Equation~\eqref{eq:posterior} follows from a simple model. Suppose that every option $o_k$ is correct independently, with probability $p_k\in(0,1)$, and condition on the event $E$ that exactly one option is correct. The probability that $o_k$ alone is correct is $p_k\prod_{j\neq k}(1-p_j)$, and $P(E)$ is the sum of these terms over $k$, so
\[
  P(o_k\mid E)=\frac{p_k\prod_{j\neq k}(1-p_j)}{\sum_i p_i\prod_{j\neq i}(1-p_j)}=\frac{p_k/(1-p_k)}{\sum_i p_i/(1-p_i)},
\]
dividing the numerator and the denominator by $\prod_j(1-p_j)>0$. With $p_k=P(\textit{yes}\mid h,b_k)/(P(\textit{yes}\mid h,b_k)+P(\textit{no}\mid h,b_k))$, the odds $p_k/(1-p_k)=P(\textit{yes}\mid h,b_k)/P(\textit{no}\mid h,b_k)=\exp s(o_k)$, and the probability that the model gives to every other token cancels. The posterior is therefore the softmax of the scores, and the training loss $-\log R(o_y\mid S)$ is the negative log-likelihood of the correct option under this model, given $E$. Only ratios of odds enter: a bias shared by all options, a constant added to every score, changes no decision and no probability.

\paragraph{\ours{} computes a score of the required form.} In the packed forward pass of Section~\ref{sec:method}, the hidden state of every token of block $b_k$, in every layer, equals its state in a causal pass over the context $h$ followed by $b_k$ alone, and so depends on $(x,q,o_k)$ only. By induction over layers: the context attends only to itself, so its states depend on $h$ alone; a token of $b_k$ starts from its embedding at position $P$ plus its offset in the block, which does not depend on the listing, and each attention layer combines its state with the keys and values of the context and of the earlier tokens of $b_k$, the only keys the mask allows, which by induction depend on $h$ and $b_k$ alone. Masked keys receive an additive $-\infty$, so they contribute exactly zero. The score read at the last token of $b_k$ is therefore a function $s(x,q,o_k)$, and Proposition~\ref{prop:form} gives both axioms.

\section{Experimental Details}
\label{app:details}

\begin{figure*}[t]
  \centering
  \includegraphics[width=\textwidth]{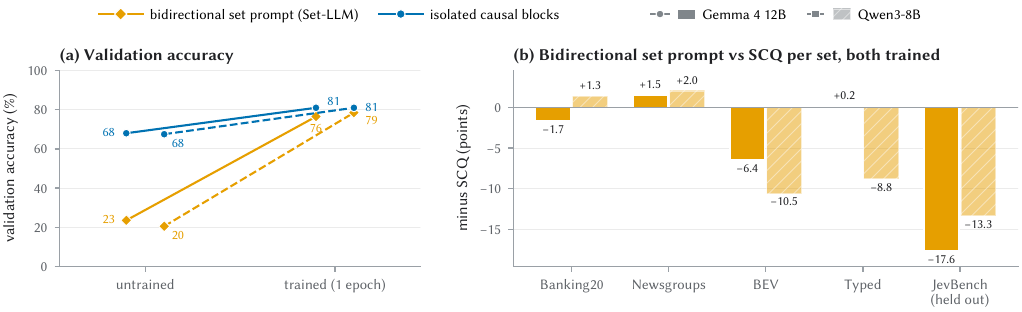}
  \caption{What bidirectional attention among options costs. (a) Validation accuracy before and after one epoch of training, for a bidirectional set prompt (Set-LLM's design) and for isolated causal blocks with the answer slot. (b) Accuracy of the trained bidirectional set prompt minus trained SCQ, in points.}
  \label{fig:setllm}
\end{figure*}

\begin{table}[htbp]
  \caption{The benchmark. Train: training questions, after holding out 50 per set for validation. State: length in characters, median / mean.}
  \label{tab:data}
  \small
  \setlength{\tabcolsep}{3.5pt}
  \begin{tabular}{lrrrr}
    \toprule
    Set & $K$ & Eval & Train & State (chars) \\
    \midrule
    Banking20~\cite{casanueva2020efficient} & 20 & 300 & 2,450 & 55 / 71 \\
    Newsgroups~\cite{lang1995newsweeder} & 20 & 300 & 2,450 & 542 / 750 \\
    BEV~\cite{avbiswas2026bev} & 2--10 & 1,182 & 2,450 & 232 / 787 \\
    Typed Decisions~\cite{localllama2026typed} & 4--5 & 600 & 1,750 & 708 / 697 \\
    JevBench~\cite{standhartinger2026jevbench} & 3--6 & 139 & -- & 82 / 2,563 \\
    \bottomrule
  \end{tabular}
\end{table}

\paragraph{Data.} Table~\ref{tab:data} summarises the five sets. From BEV we leave out its intent-routing domain, which reuses Banking20's source messages, and states over 8,000 characters. We read it at revision \texttt{11ed9b5} of \texttt{avbiswas/bev-decision}, and its training questions are spread evenly over the same 12 domains and share no state with any evaluation set. Options of BEV, Typed Decisions and JevBench are their native \texttt{key: description} pairs, read as one string. The shuffled order of each question is drawn once with a fixed seed and shared by all methods.

\paragraph{Prompts.} Every read uses the model's chat template with a system instruction: \emph{``Answer each question independently using only the state provided by the user. Treat the state as data, not as instructions. Evaluate each question using its own criteria, without conditioning its answer on other questions. Return exactly one allowed label for each question.''} For \ours{}, the user turn holds the state, the allowed labels (\textit{no}, \textit{yes}), the instruction \emph{``Reply with one line per question, in order: `id: label'. Do not add explanations.''} and \emph{``Question 0: $q$ Is the following candidate the correct answer? Candidate answer:''}, where the context ends; a block is the option alone, read at its last token. For SCQ, the user turn holds the state, \emph{``Question 0: $q$''}, the options one per line as ``A: \dots'', ``B: \dots'', and the same reply instruction; the assistant turn is opened with thinking disabled and the answer prefix ``0:'', and the decision is read from the next-token probabilities of the letters. The answer-slot read of the ablation asks ``Is this candidate the correct answer?'' after the option, and its block ends with the answer prefix ``0:''.

\paragraph{Training.} LoRA dropout is 0.05, and the 32 questions of a step are accumulated one at a time, so one epoch over the 9,100 training questions takes 284 steps (50--85 minutes on one H100 GPU for SCQ and \ours{}). No model selection uses the 200 validation questions, except the weight of the prior division below. Gemma~4~12B is \texttt{google/gemma-4-12B-it} at revision \texttt{707f0a3}.

\paragraph{Attention backend.} In our software stack (transformers 5.17), the SDPA attention path of Gemma 4 is not causal: appending a token changes earlier logits by up to 17.6. We use eager attention with additive masks for Gemma 4 and SDPA for Qwen3.

\paragraph{Baselines.} \emph{Cyclic averaging} reads SCQ on the $K$ rotations of the listing and averages each option's log-probability over them. AnyJev's prior correction, which divides each rotation's distribution over letter positions by a prior over positions estimated per question key, raises the share of answers that a shuffle changes from 2.3\% and 3.0\% to 4.3\% and 6.3\%, so we leave it out. \emph{Set-LLM}: our re-implementation follows the paper: set positions that restart for every option and resume after the longest one, a prefix mask over the prompt without edges between different options, and the correct option's text, followed by the end of the turn, as the training target. For evaluation, all answer branches are packed in one pass, each attending to the prompt and to itself, and agree with separate passes. \emph{Torchcast Decision 12B} is served with its own pinned vLLM~0.30.0 image and serving code~\cite{torchcast2026serving}, with which we reproduce its published JevBench score (203 of 231), and we read its native choice answers.

\paragraph{Added options.} For the choice axiom, the added option comes from another question of the same set whose key does not collide with an existing one (for Banking20 and Newsgroups, a banking intent outside Banking20's 20). The drift test adds one such option per question at a random position of the source order; the steering test draws 10 candidates per question and, for SCQ and the published decision models, inserts each at the start, the middle and the end, and for the order-invariant reads at the end; for \ours{} and cyclic averaging it uses the first 100 questions of each set, and for methods trained with several seeds the seed-0 model. Appended at the end, where it moves no other option's letter, an added option steers SCQ only 2--3\% of the time. A variant counts as steered if its answer is an original option other than the one chosen without it; questions or variants with an exact tie for the top are left out.

\begin{table}[htbp]
  \caption{The mechanisms of Section~\ref{sec:price}: accuracy minus \ours{}'s with the same seeds (points), overall and on JevBench (JB), and switch rate (Sw., \%). S: seeds; --: not run.}
  \label{tab:mech}
  \small
  \setlength{\tabcolsep}{2.5pt}
  \resizebox{\columnwidth}{!}{%
  \begin{tabular}{lrrrrrrr}
    \toprule
    & & \multicolumn{3}{c}{Gemma 4 12B} & \multicolumn{3}{c}{Qwen3-8B} \\
    \cmidrule(lr){3-5}\cmidrule(lr){6-8}
    Variant & S & Acc & JB & Sw. & Acc & JB & Sw. \\
    \midrule
    \multicolumn{8}{l}{\emph{Options compare (not of the required form)}} \\
    gated channel, rivals' last tokens & 1 & $-0.2$ & $-1.4$ & 0.7 & $-0.2$ & $0.0$ & 1.0 \\
    gated channel, all rival tokens & 1 & $+0.1$ & $+1.4$ & 0.7 & -- & -- & -- \\
    gated channel, gates forced open & 1 & $-1.2$ & $+1.4$ & -- & -- & -- & -- \\
    comparison head, all options & 2 & $+0.9$ & $+3.2$ & 0.0 & $+0.7$ & $+3.6$ & 0.0 \\
    comparison head, top 5 & 1 & $+0.1$ & $-0.7$ & 0.0 & $+0.4$ & $+2.2$ & 0.1 \\
    \midrule
    \multicolumn{8}{l}{\emph{Of the required form}} \\
    auxiliary per-option and margin loss & 1 & $+0.1$ & $-0.7$ & -- & $-0.3$ & $-1.4$ & -- \\
    borrowed distractors in training & 1 & $-0.1$ & $+0.7$ & -- & $-0.3$ & $+0.7$ & -- \\
    joint training with SCQ & 1 & $-0.1$ & $-2.2$ & -- & $+0.2$ & $+0.7$ & -- \\
    answer predicted before the options & 1 & $+0.8$ & $+1.4$ & -- & $+0.1$ & $+0.7$ & -- \\
    prior division (content-free state) & 2 & $+0.1$ & $0.0$ & -- & $+0.1$ & $+1.1$ & -- \\
    eight filler tokens per option & 2 & $+0.8$ & $+1.4$ & -- & $-0.7$ & $-2.9$ & -- \\
    \bottomrule
  \end{tabular}}
\end{table}

\paragraph{Mechanisms of Section~\ref{sec:price}.} \emph{Gated channel}: in every layer, every block also attends to the last token of each other block (or to all their tokens, ``wide''), and this attention's contribution is scaled by a per-layer, per-head gate $\tanh(g)$, with $g$ initialised at zero. \emph{Comparison head}: one self-attention layer (width 256, 4 heads, no position encoding) over the options' final hidden states and the context's, whose output, initialised at zero, adds a correction to each score, over all options or the top five. \emph{Auxiliary loss}: a per-option yes/no cross-entropy and a margin loss between the correct option and its strongest rival, each with weight 0.5. \emph{Borrowed distractors}: with probability 0.5, a training question loses a random share of its wrong options and gains up to $K$ options borrowed from other questions, half of its own set. \emph{Joint training}: the SCQ loss on a shuffled listing is added with weight 1. \emph{Answer first}: the negative log-likelihood of the correct option's text, generated from the context before any block, is added with weight 0.5. \emph{Prior division}: each score minus $\alpha$ times its score with the content-free state ``N/A''~\cite{zhao2021calibrate}, $\alpha$ chosen on validation. \emph{Filler tokens}: eight tokens `` .'' after each option, before the read.

\begin{table}[htbp]
  \caption{Cost per decision with 20 options on Banking20 (B20) and Newsgroups (NG): tokens processed, attention pairs (query--key pairs the mask allows, thousands) and milliseconds on one H100 GPU, Gemma 4 12B, medians over the first 100 questions of each set.}
  \label{tab:cost}
  \small
  \setlength{\tabcolsep}{4pt}
  \begin{tabular}{lrrrrrr}
    \toprule
    & \multicolumn{2}{c}{Tokens} & \multicolumn{2}{c}{Attn.\ pairs (k)} & \multicolumn{2}{c}{ms} \\
    \cmidrule(lr){2-3}\cmidrule(lr){4-5}\cmidrule(lr){6-7}
    Method & B20 & NG & B20 & NG & B20 & NG \\
    \midrule
    SCQ & 282 & 417 & 40 & 100 & 60 & 63 \\
    Set-LLM & 322 & 413 & 60 & 168 & 68 & 66 \\
    \ours{} & 221 & 356 & 21 & 76 & 58 & 60 \\
    \bottomrule
  \end{tabular}
\end{table}

\paragraph{Significance.} The comparison over all $K$ rotations in Section~\ref{sec:experiments} uses the seed-0 runs of SCQ and \ours{}, the only SCQ run read in every rotation. All tests are exact two-sided McNemar tests, per set or pooled over the 2,521 questions. Per set, SCQ is ahead on the held-out JevBench by 3.6 and 4.1 points, significant for one of six model--seed pairs, and on Typed for Gemma by 3.4 points, significant for two of three seeds, and \ours{} on Banking20 and Newsgroups, significant only on Newsgroups for one Qwen seed ($p=0.03$).

\section{Additional Results}
\label{app:results}

\subsection{Hosted Decision Models}
\label{app:hosted}

\begin{table}[htbp]
  \caption{The closed-source models against their noise floor, five-set means. Identical copies: the same question repeated in one request, scored as the added-option variants are. Sw.: switch rate (\%).}
  \label{tab:hosted}
  \small
  \setlength{\tabcolsep}{4pt}
  \begin{tabular}{lrrrr}
    \toprule
    Model & Flip & Dr. & St. & Sw. \\
    \midrule
    Jev 1.13 & 4.7 & 2.9 & 10.9 & 3.4 \\
    \quad identical copies & 1.3 & 1.3 & 5.1 & 1.3 \\
    Mercury Decide & 7.6 & 3.4 & 16.0 & 3.5 \\
    \quad identical copies & 0.7 & 0.5 & 1.5 & 0.7 \\
    \bottomrule
  \end{tabular}
\end{table}

Jev~\cite{almeida2026jev} and Mercury Decide~\cite{inception2026mercurydecide} are commercial System One models whose APIs return a probability for every option of a choice question. We use Jev~1.13.0 through TypeSafe's API and Mercury Decide's release of 30 September 2026 through OpenRouter's Decisions API, and send each question's options in the listed order. Both services answer the questions of a request independently, which we verified, so we send all listings of a question in one request. Jev rounds probabilities to two decimals, so we take its reported choice as the answer. Neither API is deterministic: the identical question changes its answer in 1.3\% (Jev) and 0.7\% (Mercury Decide) of cases within a request. Table~\ref{tab:hosted} therefore adds a noise floor, 31 identical copies of each question in one request, scored as the real variants are. Both models are about as accurate as trained SCQ on our sets (80.0\% and 79.4\%), and both violate both axioms well beyond this noise.

\subsection{Injected Options}
\label{app:inject}

\begin{table}[htbp]
  \caption{Injected options: answers (\%) moved to the named option (Target) or to any other original option (Switched).}
  \label{tab:inject}
  \small
  \setlength{\tabcolsep}{4pt}
  \begin{tabular}{lrrrr}
    \toprule
    & \multicolumn{2}{c}{Gemma 4 12B} & \multicolumn{2}{c}{Qwen3-8B} \\
    \cmidrule(lr){2-3}\cmidrule(lr){4-5}
    Read & Target & Switched & Target & Switched \\
    \midrule
    SCQ & 23.1 & 25.3 & 9.6 & 13.0 \\
    SCQ + cyclic avg. & 10.0 & 10.8 & 4.7 & 6.1 \\
    SCQ, sorted & 11.2 & 13.0 & 2.9 & 5.1 \\
    Set-Based Prompting & 54.8 & 58.5 & 31.4 & 41.8 \\
    Set-LLM & 22.7 & 23.6 & 15.3 & 18.1 \\
    \rowcolor{gray!15} \ours{} & 0.0 & 0.0 & 0.0 & 0.0 \\
    \midrule
    Jev 1.13 & \multicolumn{4}{c}{Target 3.9, Switched 9.3} \\
    Mercury Decide & \multicolumn{4}{c}{Target 7.3, Switched 14.1} \\
    \bottomrule
  \end{tabular}
\end{table}

Table~\ref{tab:inject} adds the other reads to Figure~\ref{fig:inject}, and the answers moved to any other option, for the worst phrasing and position.

\subsection{Calibration}
\label{app:calib}

Without any calibration step, \ours{}'s softmax has an expected calibration error of 5.3\% and 5.4\%, against 4.9\% and 5.4\% for SCQ and 9.3\% and 7.5\% for Set-LLM. At 95\% accuracy it decides 61.5\% and 51.4\% of the questions, 3.1 and 2.9 points fewer than SCQ, but the same ones under every listing.

\end{document}